\documentclass[letterpaper, 10 pt, conference]{ieeeconf}

\IEEEoverridecommandlockouts                              %
\usepackage[utf8]{inputenc}
\usepackage{amsmath}

\usepackage{amsfonts}
\usepackage{amsthm}
\usepackage{amssymb}
\usepackage{graphicx}

\usepackage{xcolor}

\title{An Optimization Approach for a Robust and Flexible Control in Collaborative Applications} 
\author{Federico Benzi and Cristian Secchi
\thanks{F. Benzi and C. Secchi are
    with the Department of Sciences and Methods of Engineering,
    University of Modena and Reggio Emilia, Italy  {\tt\small
      \{federico.benzi,cristian.secchi\}@unimore.it}}}
      
\begin{document}

\maketitle
\begin{abstract}
    
    In Human-Robot Collaboration, the robot operates in a highly dynamic environment. Thus, it is pivotal to guarantee the robust stability of the system during the interaction but also a high flexibility of the robot behavior in order to ensure safety and reactivity to the variable conditions of the collaborative scenario.\\
    In this paper we propose a control architecture capable of maximizing the flexibility of the robot while guaranteeing a stable behavior when physically interacting with the environment. This is achieved by combining an energy tank based variable admittance architecture with control barrier functions. The proposed architecture is experimentally validated on a collaborative robot.

\end{abstract}
\section{Introduction}
\label{sec:intro}

In order to implement an effective human robot collaboration, it is fundamental to guarantee that the robot can adapt online its behavior in a stable and efficient way. A lot of work in this direction has been done using different techniques such as task redundancy \cite{Su2019RAL,Zanchettin2016TASE}, adaptive physical interaction \cite{CTLandi2019IJRR,Dimeas2016ToH} and coaching \cite{Lan19electronics,Kar17icra} to name a few. Nevertheless, at the best of the authors' knowledge, a control architecture that allows to address the problems of robust stability and flexibility both in free motion and during the interaction is still missing.

The behavior of a robot can be defined by specifying the way it physically interacts with the environment and the set of (dynamic) constraints it has to satisfy. Adapting the interaction implies to dynamically vary the parameters  (e.g. stiffness, inertia) that determine the way forces are exchanged with the environment. Dynamically constraining the motion of the robot allows to dynamically shape its behavior.

Admittance control \cite{Villani2016} is a very popular strategy for controlling the interaction, and it has been widely used  in collaborative  scenarios  (see, e.g., \cite{Dimeas2016ToH, CTLandi2019IJRR,Ayd18toh}). It enforces the robot to reproduce a desired passive physical behavior (e.g. mass-spring-damper), i.e. the admittance dynamics. Because of the passivity of the admittance dynamics, the behavior of the robot is robustly stable, i.e. stable both in free motion and when interacting with a poorly known environment (see e.g.\cite{Secchi2007}).

When using admittance control, the interaction dynamics is determined by the choice of the parameters of the admittance dynamics (e.g. inertia and stiffness) and a specific set of parameters may not be suitable for all stages of a given application \cite{Dimeas2016ToH}. 

This severely limits the flexibility of the controller, as changing online the dynamic parameters can lead to the loss of passivity \cite{Ferraguti2015TRO} and, consequently, to a possibly unstable behavior in human robot collaboration as shown in \cite{CTLandi2019IJRR}. 

Several strategies for implementing a variable admittance control have been proposed (see e.g. \cite{CTLandi2019IJRR,Grafakos2016SMC}) but they allow only a limited  variation of the interaction dynamics while significantly complicating the overall control architecture

In \cite{Secchi19ICRA}, an optimization framework has been proposed in order to implement a variable admittance controller in a robust and flexible way. Here, energy tanks \cite{Ferraguti2015TRO,Franken2011TRO} have been exploited for disembodying passivity from  a specific physical behavior. The energy stored in the tank has been exploited as a constraint to be satisfied for passively implementing any desired admittance dynamics.  This allows to achieve the maximal flexibility while satisfying the passivity constraints, i.e. robust stability.
Yet \cite{Secchi19ICRA} does not differentiate between the interaction phase and the free motion of the robot, and the admittance control is active in both situations. In a collaborative scenario, the robot often switches between free motion and interaction and constantly requiring the satisfaction of both the passivity constraint and the continuous implementation of a specific admittance dynamics may affect the performance of the robot in free motion. 

Moreover, robust stability is not enough to guarantee the complete safety and the flexibility of a robot in a shared workspace: additional time-varying constraints, such as obstacle and human avoidance, self collision and bounding the joint position, need to be satisfied for ensuring a safe and flexible behavior. 
Control barrier functions (CBFs) \cite{Ame17tac,Ame19ecc} have been successfully exploited in many robotics applications for forcing the robot to remain in a desired subset of the state space, i.e. to dynamically constrain the behavior of the robot (see e.g. \cite{Not18ral,Zha20toh,Fu20tnse}). Time Varying CBFs \cite{Not20tcst} allow to consider time-varying constraints and they have been successfully exploited for enforcing safety in human robot collaboration \cite{Fer20ral,Fer20ras}. Furthermore, tasks to be executed can be encoded as CBFs as shown in  \cite{notomista2020settheoretic}. The control input for satisfying the dynamic constraints modeled by the CBFs can be found by solving a convex optimization problem, which makes the CBFs amenable to real-time implementation. 

Passivity can be encoded using Control Barrier Functions \cite{Not19mrs} but no interaction controllers are available yet. 

 In this paper we propose a novel control architecture that merges energy tanks, for implementing a flexible and passive interactive behavior, and control barrier functions, for dynamically constraining the behavior of the robot. 
 The proposed architecture leverages control barrier functions to encode a series of constraints to be satisfied and energy tanks for passively implementing a variable admittance controller. The control input is the solution of a single convex optimization problem that considers both the passivity constraint and the constraints due to the control barrier functions.
  Using the proposed architecture it will be possible to make a collaborative robot extremely flexible both during free motion and during the interaction with possibly unknown environments.
 
 The contributions of this paper are:
 \begin{itemize}
     \item A convex optimization problem merging the passivity constraint and control barrier functions constraints;
     \item A control architecture capable of maximizing the flexibility of the robot both during the interaction and during free motion;
     \item An experimental validation of the proposed architecture on a collaborative robot.
 \end{itemize}
 
This paper is organized as follows: in Sec. \ref{sec:problem} the main problem addressed in this paper is formulated. In Sec. \ref{sec:tanks} energy tanks, their usage in passivity preservation and the overall tank-based architecture is presented. In Sec. \ref{sec:CBF} the CBF-based architecture for flexible task execution is described, while in Sec. \ref{sec:admittance} the collaborative constraint-oriented control architecture is shown. In Sec. \ref{sec:experiments} the proposed architecture is validated onto a robotic setup, and finally in Sec. \ref{sec:conclusions} conclusions are drawn and future work is addressed. 

\section{Problem Formulation}
\label{sec:problem}
\begin{figure}
	\centering
	\includegraphics[width=\columnwidth]{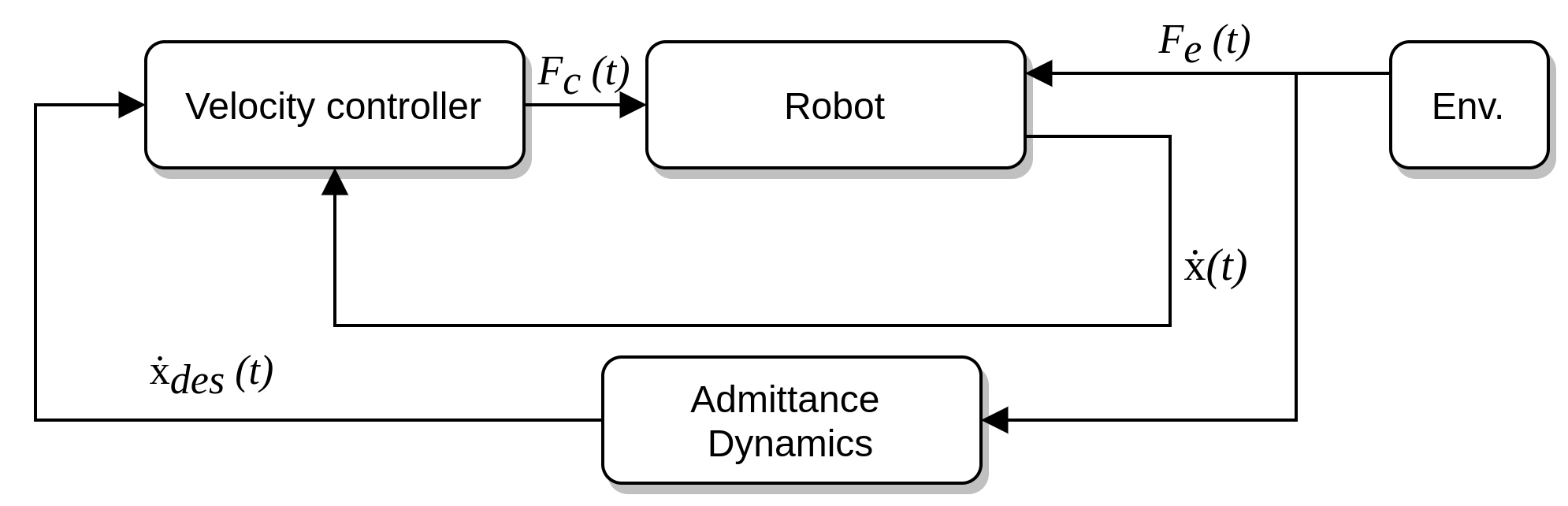}
	\caption{The admittance control architecture. The input $F_c$ represent the control force of the low level controller actuating the robot.}
	\label{fig:admittance architecture}
\end{figure}
Consider a velocity controlled fully actuated $n$-DOFs manipulator. The model of the robot can be described by:
 
\begin{equation} \label{eq:robot kinematic model}
        \dot x = J(q)u 
\end{equation}
where  $x \in \mathbb{R}^{m}$ is the pose of the end-effector and $q \in \mathbb{R}^{n}$ is the vector of joint variables. $J(q) \in \mathbb{R}^{m \times n}$ is the Jacobian of the robot and $u \in \mathbb{R}^{n}$ is the joint velocity input.

Thanks to the low-level velocity controller, a desired velocity profile $\dot x_{des}\in\mathbb{R}^m$ can be reproduced by  simply setting $u=J^+(q)\dot{x}_{des}$, where $J^+(q)$ denotes the pseudo-inverse of the Jacobian $J(q)$, such that $\dot x (t)\approx \dot{x}_{des}$.

The interaction with the environment is regulated by the admittance control architecture reported in Fig.~\ref{fig:admittance architecture}. In order to enforce a desired dynamic behavior for the robot, the interaction force $F_e(t)\in\mathbb{R}^m$ is integrated by the admittance dynamics that produces a desired velocity $\dot{x}_{a}$. 
By setting $\dot{x}_{des}(t)=\dot{x}_a \approx \dot x$, the robot exactly reproduces the admittance dynamics. Thus, if the admittance dynamics is passive, the controlled robot behaves as a passive system and, therefore, it is stable both in free motion and in contact with the environment, i.e. robustly stable. 


Admittance dynamics model mechanical systems and they can be represented by the following Euler-Lagrange model
\begin{equation}\label{eq: admittance model}
    M(x_a)\ddot{x}_a + C(x_a,\dot{x}_a)\dot{x}_a + D(x_a)\dot{x}_a + \dfrac{\partial P}{\partial x_a}= F_{e}
\end{equation}
where $x_a\in\mathbb{R}^m$ is the pose and  $M(x_a) = M^{T}(x_a) > 0$ is the inertia matrix, $C(x_a,\dot{x}_a)$ represents the Coriolis effects, $D(x_a) \geq 0$ is a damping matrix and $P: \mathbb{R}^{m} \rightarrow \mathbb{R}$ is a potential field active on the system (e.g. elastic potential).\\
As well known (see, e.g., \cite{Secchi2007}), \eqref{eq: admittance model} is passive with respect to the pair $(\dot{x}_a(t),F_{e}(t))$, using the following storage function:
\begin{equation}\label{eq: admittance energy function}
    H_{a}(x_a, \dot{x}_a) = P(x_a) + \frac{1}{2}\dot{x}_a^{T}M(x_a)\dot{x}_a
\end{equation}

Unfortunately, when considering a variable admittance, i.e. such that the physical parameters characterizing the desired dynamics are time-varying, passivity can be lost. In fact, considering the time-varying version of \eqref{eq: admittance model} 
\begin{equation}\label{eq:tvaradm}
    M(x_a,t)\ddot{x}_a + C(x_a,\dot{x}_a,t)\dot{x}_a + D(x_a,t)\dot{x}_a + \dfrac{\partial P}{\partial x_a}(t)= F_{e}
\end{equation}
we have that 
\eqref{eq: admittance energy function} becomes then
\begin{equation} \label{eq: time-varying energy function}
    H_{a}(x_a,\dot{x}_a,t) = P(x_a,t) + \frac{1}{2}\dot{x}_a^{T}M(x_a,t)\dot{x}_a
\end{equation}
whence, using \eqref{eq:tvaradm}, we can write
\begin{equation} \label{eq: time-varying admittance power function}
    \dot{H}_{a}(x_a,\dot{x}_a,t) = F_{e}^{T}\dot{x}_a -\dot{x}^{T}_aD(x_a,t)\dot{x}_a + \dfrac{\partial H_{a}}{\partial t}
\end{equation}
where the last term is sign indefinite and can introduce energy into the system and, therefore, lead to a loss of passivity. 

Thus, standard admittance control is not suitable for implementing an interaction that is flexible and passive at the same time. 

Besides passivity, various dynamic and time-varying constraints are necessary for establishing safety and other desired behaviors for the robot. These constraint cannot be simply embedded in the standard admittance control architecture.

We aim at developing a control architecture that allows to passively implementing a time varying admittance dynamics as the one in \eqref{eq:tvaradm}  and to enforce a set of dynamic and time-varying constraints on the robot. 





\section{Energy Tank Architecture}
\label{sec:tanks}



In this section we show how to exploit energy tanks \cite{Ferraguti2015TRO},\cite{Franken2011TRO} for reproducing a variable admittance dynamics \eqref{eq:tvaradm} on the robot \eqref{eq:robot kinematic model} following the approach presented in \cite{Secchi19ICRA}.


Energy tanks are energy storing elements which can be represented by:
\begin{equation} \label{eq:tank description}
    \begin{cases}
        \dot{x}_{t} = u_{t}\\
        y_{t} = \frac{\partial{T}}{\partial{x_{t}}} = x_{t}(t)
    \end{cases}
\end{equation}
where $x_{t}\in \mathbb{R}$ is the state of the tank, while the pair $(u_{t},y_{t}) \in\mathbb{R}\times\mathbb{R}$ represents the power port the tank can exchange energy through and
\begin{equation} \label{eq:tank energy}
    T(x_t)= \frac{1}{2}x^{2}_{t}
\end{equation}
is the function representing the energy stored in the tank. 

Energy tanks are completely disconnected from a specific physical dynamics. Thus, the energy stored in the tank can be exploited for implementing any desired behavior.
This can be achieved by interconnecting the power port of the tank $(u_t,y_t)$ with the power port $(F_e,\dot x_{des})$ of the implemented admittance dynamics as:
\begin{equation} \label{eq:tank modulation}
    \begin{cases}
    u_{t}(t) = A^{T}(t)F_e(t)\\
    \dot x_{des}(t) = A(t)y_{t}(t)
    \end{cases}
\end{equation}

where $A(t)  \in \mathbb{R}^{n}$ is defined as
\begin{equation} \label{eq:modulation matrix description}
    A(t) = \frac{\gamma(t)}{x_{t}(t)}
\end{equation}
and $\gamma(t) \in \mathbb{R}^{n}$ is the desired value for the output $\dot{x}_{des}(t)$, i.e. the velocity implementing the desired admittance. In fact, if we plug \eqref{eq:tank modulation} into \eqref{eq:tank description} we get
\begin{equation} \label{eq:tank modulated description}
    \begin{cases}
        \dot{x}_{t} = A^{T}(t)F_e(t)\\
        \dot{x}_{des} = A(t)y_{t}(t)= \gamma(t)
    \end{cases}
\end{equation}
which implies that any desired port behavior can be obtained by an appropriate modulation of the energy exchanged with the tank. From \eqref{eq:tank description},\eqref{eq:tank energy} and \eqref{eq:tank modulated description} it follows that
\begin{equation}\label{eq:tankadmbalance}
\dot{T}=u_t y_t=A^T (t)F_e y_t=\gamma^T F_e
\end{equation}

which clearly shows that the energy needed for implementing the desired behavior is extracted from/injected into the tank.

The modulated tank \eqref{eq:tank modulated description} faces a singularity whenever $x_{t}(t)=0$ because of the definition of the modulation factor in \eqref{eq:tank modulated description}. The singularity happens when the  tank is empty and no behavior can be implemented using the energy stored in the tank \cite{Ferraguti2015TRO}. In order to avoid this condition, it is necessary to  initialize $x_{t}$ such that $T(x_{t}(0)) \geq \underline \varepsilon > 0$ and to ensure that $T(x_{t}(t)) \geq \underline \varepsilon$ $\forall t>0$. 
As shown in  \cite{Secchi19ICRA}, the following result holds: 
\newtheorem{prop}{Proposition}
 \begin{prop} \label{prop: always passive}
 	If $T(x_{t})\geq \underline{\varepsilon}$ for all $t\geq0$, then the modulated tank \eqref{eq:tank modulated description} remains passive independently of the desired output $\gamma(t)$.
 \end{prop}
Thus, as long as the tank is not depleted, any desired port behavior can be passively implemented by modulating the energy stored in the tank.
Furthermore, as shown in \cite{Giordano2013IJR} for the single port case and in \cite{Riggio2018} for the multi-port case, any passive dynamics can be reproduced without depleting the energy tank. Thus, modulated energy tanks can be used as a generalized admittance control since they allow to reproduce any physical dynamics.

In case we aim at reproducing a non passive admittance dynamics, i.e. the time varying dynamics in \eqref{eq:tvaradm}, we can exploit Prop.~\ref{prop: always passive} for enforcing a passive implementation of the desired behavior. In fact, it is possible to guarantee the passivity of the modulated tank by enforcing the following constraint:

 \begin{equation}\label{eq: passivity constraint}
     T(x_{t}) \geq \underline{\varepsilon} \quad \forall t\geq 0
 \end{equation}
Thus, as shown in \cite{Secchi19ICRA}, it is possible to find the best passive implementation of any desired admittance dynamics by solving the following optimization problem

 \begin{equation} \label{eq: opt problem ICRA19}
\begin{aligned}
& \underset{\dot{x}_{des}}{\text{minimize}}
& & ||\dot{x}_{des} - \dot{x}_{a}||^{2} \\
& \text{subject to}
& & \int_{0}^{t}F_{e}^{T}(\tau)\dot{x}_{des}(\tau)d\tau\geq - T(x_{t}(0)) + \underline{\varepsilon}
\end{aligned}
\end{equation}
 The problem \eqref{eq: opt problem ICRA19} allows to find
the best passive approximation of the desired behavior $\dot{x}_{a}$, taking into account the amount of energy stored in the tank. The solution $\dot{x}_{des}$  is then utilized to correctly tune the modulation matrix $A(t)$ in \eqref{eq:tank modulated description}. In this way, a passive energy balance is guaranteed even if the admittance parameters change online.\\
 In \cite{Secchi19ICRA}, a discrete time version of the problem was introduced, which turns the integral constraint in \eqref{eq: opt problem ICRA19} into a linear one. The resulting optimization problem is then convex, which makes it computationally fast and simple in formulation, thus suitable for being executed in real time and embedded into a larger optimization-based framework.
\section{Control Barrier Functions based architecture}
\label{sec:CBF}

In this section we show how to exploit CBFs for enforcing multiple time-varying dynamic constraints on a robot described by \eqref{eq:robot kinematic model}. The presented framework leverages the techniques exposed in \cite{Notomista2018CoRR}, \cite{Notomista2019CoRR}, where 
both kinematic limits and desired tasks to be executed by the robot can be represented as time-varying constraints on the control input. Thus, it is possible to model tasks and physical limitations  as constraints to be fulfilled by the controlled system and CBFs are exploited for formulating an optimization problem for finding the best input that satisfies all the constraints, e.g. that can lead to the optimal execution of all the tasks.


We consider tasks whose execution can be expressed as the minimization of a non negative, possibly time-varying, continuously-differentiable cost function $C : \mathbb{R}^{n} \times \mathbb{R} \rightarrow \mathbb{R}$. Using the robot model in \eqref{eq:robot kinematic model} and considering as an output variable the time-varying task variable $\sigma \in \mathbb{R}^{n}$, this can be described by the following optimization problem:

\begin{equation} \label{eq: prob cost function}
\begin{aligned}
& \underset{u}{\text{minimize}}
& & C(\sigma,t) \\
& \text{subject to}
& & \dot{x} = J (q)u \\
&&& \sigma = k(x,t)
\end{aligned}
\end{equation}
The optimization problem can be reformulated and made computationally simple using Control Barrier Functions. 
Let $\mathcal{C}\subset \mathbb{R}^n$ be the subset where the task is considered executed, i.e. where $C(\sigma,t)=0$. Consider the a control barrier function $h:\mathbb{R}^{n} \times \mathbb{R} \rightarrow \mathbb{R}$ defined as $h(\sigma, t)=-C(\sigma,t)$. By construction $h$ is non negative only in the region of satisfaction of the task, i.e. when $C(\sigma,t)=0$. Thus, enforcing the non negativity of $h$ is equivalent to enforcing the execution of the task $\sigma$. This can be encoded in the following convex optimization problem \cite{Notomista2018CoRR}:
\begin{equation} \label{eq: prob CBF general}
\begin{aligned}
& \underset{\dot{q}}{\text{minimize}}
& & ||\dot{q}||^{2} \\
& \text{subject to}
& & \frac{\partial h}{\partial t} + \frac{\partial h}{\partial \sigma}\frac{\partial \sigma}{\partial x}J(q)\dot{q} + \alpha(h(\sigma,t)) \geq 0
\end{aligned}
\end{equation}
where $\alpha(\cdot)$ is an extended class $\mathcal{K}$ function\footnote{An extended class $\mathcal{K}$ is a function $\phi:\mathbb{R}\rightarrow\mathbb{R}$ such that $\phi$ is strictly increasing and $\phi (0)=0$} and where we have considered that the input of \eqref{eq:robot kinematic model} is $u=\dot q$. 
The formulation can easily be extended in order to enable the simultaneous execution of multiple tasks. Let us consider a set of M different tasks $T_{1}, \dots , T_{M}$ which have to be executed, each encoded respectively by the cost functions $C_{1}, \dots, C_{M}$. The execution of these tasks can then be accomplished by solving the following convex optimization problem:
\begin{equation} \label{eq: prob CBF multitask slack}
\begin{aligned}
& \underset{\dot{q}, \delta}{\text{minimize}}
& & ||\dot{q}||^{2} + l||\delta||^{2} \\
& \text{subject to}
& & \frac{\partial h_{m}}{\partial t} + \frac{\partial h_{m}}{\partial \sigma}\frac{\partial \sigma}{\partial x}J(q)\dot{q} \\
&&& + \alpha (h_{m}(\sigma,t)) \geq -\delta_{m} \quad  m \in  \{ 1, \dots, M\}
\end{aligned}
\end{equation}
in which $h_{m}(\sigma,t) = -C_{m}(\sigma,t)$ and $\delta = [\delta_{1},\dots,\delta_{M}]^{T}$ is the vector of slack variables corresponding to each constraint, while $l \geq 0$ is a scaling factor. Each $\delta_{i}$ indicates how much the constraint corresponding to the $i-$th task can be relaxed, in order to guarantee the feasibility of the problem even if conflicting constraints are active at the same time.

\section{Collaborative Constraint-Oriented Control Architecture}
\label{sec:admittance}

In this section we propose a collaborative constraint-oriented control architecture that merges the flexible energy management architecture illustrated in Sec.~\ref{sec:admittance} with the CBFs based architecture for task execution shown in Sec.~\ref{sec:CBF} in order to obtain a control strategy that guarantees a maximally flexible and robustly stable behavior of the robot. We will keep the exposition of the algorithm in continuous time, but the overall discussion can be easily extended to discrete time (see \cite{Secchi19ICRA} in particular for the constraint on passivity).


It is possible to enforce a passive behavior of the robot controlled using control barrier functions by inserting the passivity constraint of \eqref{eq: opt problem ICRA19} in \eqref{eq: prob CBF multitask slack} and by using the output of the optimization problem for modulating the matrix $A(t)$ in \eqref{eq:tank modulated description}. This leads to:

\begin{equation} \label{eq: prob CBF multitask slack prio passive}
\begin{aligned}
& \underset{\dot{q}, \delta}{\text{minimize}}
& & ||\dot{q}||^{2} + l||\delta||^{2} \\
& \text{subject to}
& & \frac{\partial h_{m}}{\partial t} + \frac{\partial h_{m}}{\partial \sigma}\frac{\partial \sigma}{\partial x}J(q)\dot{q} \\
&&& + \alpha(h_{m}(\sigma,t)) \geq -\delta_{m} \quad  m \in  \{ 1, \dots, M\}\\
&&& \int_{0}^{t}F_{e}^{T}(\tau)J(q)\dot{q}(\tau)d\tau \geq -T(x_{t}(0)) + \underline{\varepsilon}\\
\end{aligned}
\end{equation}
and the tank is modulated by setting $\gamma=J(q)\dot{q}$.

Notice that, since the passivity constraint is linear, the overall optimization problem maintains its convexity. Furthermore, the passivity constraint in \eqref{eq: prob CBF multitask slack prio passive}   does not present a dedicated slack variable to relax it. This is due to the fact that the preservation of passivity is of top priority as it is connected to robust stability. At the same time, this choice does not over-constrain the system during free motion, as demonstrated in the following proposition.


\newtheorem{prop3}{Proposition}
\begin{prop} \label{prop: passivity constraint in free motion}
	If $T(x_{t}(t))\geq \underline{\varepsilon}$ during the interaction, then the passivity constraint in \eqref{eq: prob CBF multitask slack prio passive} is automatically satisfied also in free motion.
\end{prop}
\begin{proof}
    Assume that free motion starts at time $\overline{t}>0$. We can then subdivide the integral term of the passivity constraint in \eqref{eq: prob CBF multitask slack prio passive} in two parts:
    \begin{multline}\label{eq: passivity constraint time-divided}
        \int_{0}^{\overline{t}}F_{e}^{T}(\tau)J(q)\dot{q}(\tau)d\tau +\int_{\overline{t}}^{t}F_{e}^{T}(\tau)J(q)\dot{q}(\tau)d\tau \geq\\ \geq  -T(x_{t}(0)) + \underline{\varepsilon}
    \end{multline}
    Since $F=0$ after $\overline{t}$, \eqref{eq: passivity constraint time-divided} reduces to
    \begin{equation} \label{eq: passivity constraint reduced}
        \int_{0}^{\overline{t}}F_{e}^{T}(\tau)J(q)\dot{q}(\tau)d\tau \geq -T(x_{t}(0)) + \underline{\varepsilon}
    \end{equation}
    which is equivalent to the constraint active during the interaction phase, thus concluding the proof.
\end{proof}
\begin{figure}
	\centering
	\includegraphics[width=\columnwidth]{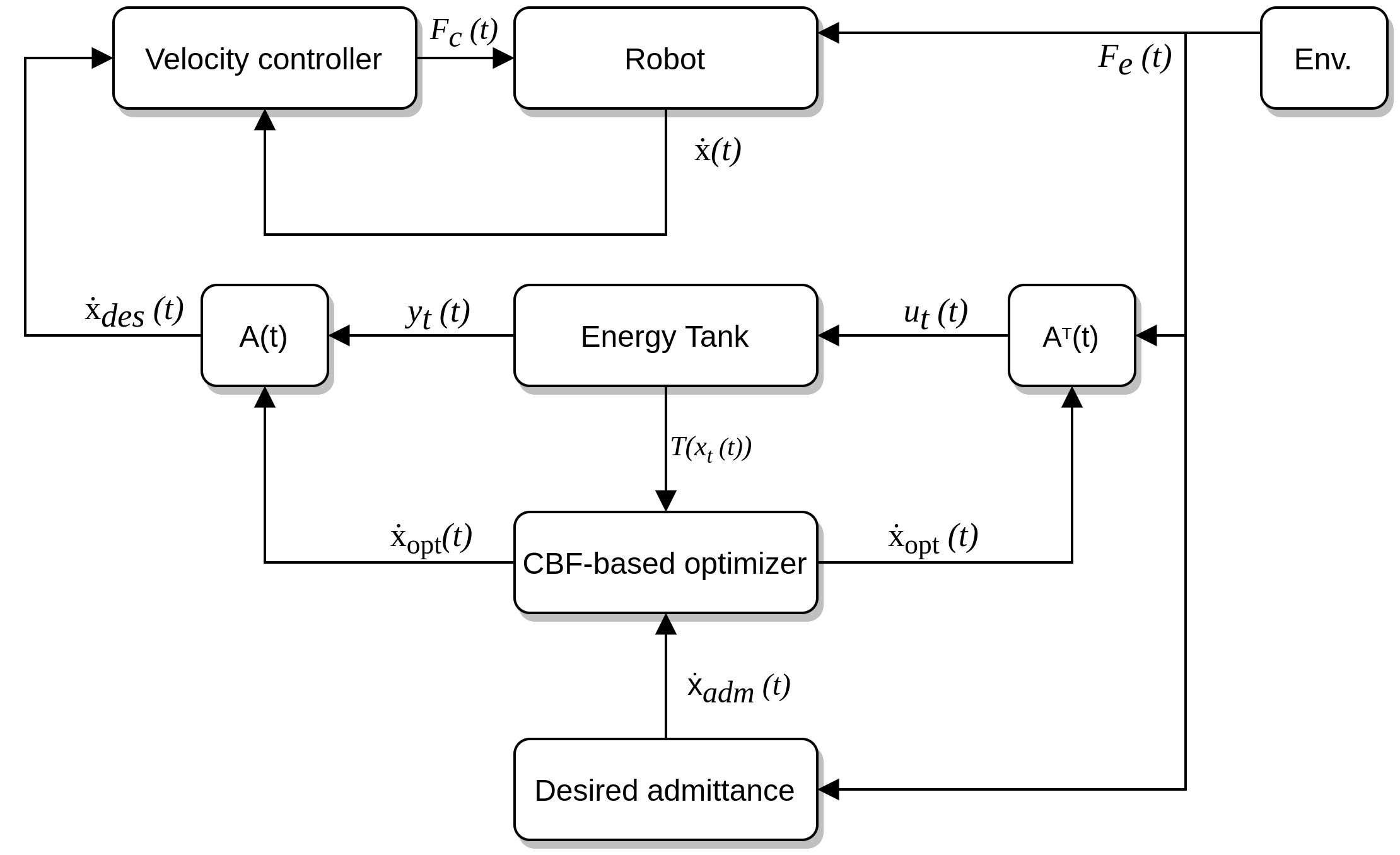}
	\caption{The final collaborative constrained-control architecture.}
	\label{fig:full architecture}
\end{figure}
Hence, the passivity constraint is automatically satisfied during free motion, meaning that the optimizer disregards it completely during this phase, focusing on satisfying the remaining $M$ constraints in the stack.\\
Following this procedure, we have therefore obtained a control architecture which is capable not only of accomplishing concurrent constraints in a flexible way but also to guarantee a robust behavior while interacting with a poorly known environment, thanks to the preservation of passivity. The developed framework acts as a sort of "armor", protecting the robot against unstable behaviors during the interaction, while still allowing it to implement all the encoded tasks at the best of its dexterity. Finally, everything is encompassed into a single convex optimization problem, for a fast and efficient resolution.

In the same fashion, we can insert the desired admittance into our own optimization problem, by properly modifying the previous objective function in \eqref{eq: prob CBF multitask slack prio passive} as follows:
\begin{equation} \label{eq: prob CBF multitask slack prio passive adm}
\begin{aligned}
& \underset{\dot{q}, \delta}{\text{minimize}}
& & ||\dot{q} - \dot{q}_{a}||^{2} + l||\delta||^{2} \\
& \text{subject to}
& & \frac{\partial h_{m}}{\partial t} + \frac{\partial h_{m}}{\partial \sigma}\frac{\partial \sigma}{\partial x}J(q)\dot{q} \\
&&& + \alpha(h_{m}(\sigma,t)) \geq -\delta_{m} \quad  m \in  \{ 1, \dots, M\}\\
&&& \int_{0}^{t}F_{e}^{T}(\tau)J(q)\dot{q}(\tau)d\tau \geq -T(x_{t}(0)) + \underline{\varepsilon}\\
\end{aligned}
\end{equation}
 in which $\dot{q}_{a} = J(q)^{+}\dot{x}_{a}$ is the desired admittance expressed in the joint space.
 In this way, we have encompassed the whole variable admittance control scheme into our task execution architecture, allowing for a flexible interaction with the environment, while still ensuring the overall passivity via the energetic constraint, as well as satisfying the other time-varying constraints in the stack.

The implementation of the desired admittance can be enforced at the price of relaxing the tasks encoded by the CBFs. A possible strategy for achieving this goal is to dynamically weight the first term in the function to optimize in \eqref{eq: prob CBF multitask slack prio passive adm}. A possible solution is to use the following metric matrix $W(t)=(1+\kappa \|F_e(t)\|^2)I_n$, where $I_n$ is the identity matrix of order $n$ and $\kappa>0$ that can be used for tuning the effect of the external force. Thus, replacing $\|\dot{q}-\dot{q}_a\|^2$ in \eqref{eq: prob CBF multitask slack prio passive adm} with 
\begin{equation}\label{eq:wnorm}
\|\dot{q}-\dot{q}_a\|^2_W(t)=(\dot{q}-\dot{q}_a)^TW(t)(\dot{q}-\dot{q}_a)
\end{equation}
allows to penalize the deviation from the desired admittance dynamics the more, the higher the interaction force is. In free motion \eqref{eq:wnorm} reduces to the standard euclidean norm used in \eqref{eq: prob CBF multitask slack prio passive adm}.

 
 The final formulation of the optimization problem is as follows:
\begin{equation} \label{eq: prob CBF multitask slack prio passive adm penalty}
\begin{aligned}
& \underset{\dot{q}, \delta}{\text{minimize}}
& & ||\dot{q} - \dot{q}_{adm}||^{2}_{W} + l||\delta||^{2} \\
& \text{subject to}
& & \frac{\partial h_{m}}{\partial t} + \frac{\partial h_{m}}{\partial \sigma}\frac{\partial \sigma}{\partial x}J(q)\dot{q} \\
&&& + \alpha(h_{m}(\sigma,t)) \geq -\delta_{m} \quad  m \in  \{ 1, \dots, M\}\\
&&& \int_{0}^{t}F_{e}^{T}(\tau)J(q)\dot{q}(\tau)d\tau \geq -T(x_{t}(0)) + \underline{\varepsilon}\\
\end{aligned}
\end{equation}
The overall control architecture is reported in Fig.~\ref{fig:full architecture}, in which we defined $\dot{x}_{opt} = J(q)\dot{q}$ as the optimal velocity in the task space.

The velocity controlled robot is wrapped by the modulated tank that is exploited for ensuring the passivity of the controlled system. The CBF-based optimizer receives as an input the desired admittance velocity $\dot x_a$ and, by solving \eqref{eq: prob CBF multitask slack prio passive adm penalty} determines the best value $\dot{x}_{opt}$ for passively implementing the desired admittance behavior and all the tasks that are encoded by the  CBFs. The matrix $A(t)$ is modulated by setting $\gamma=\dot{x}_{opt}$.


\section{Experiments}
\label{sec:experiments}
The framework proposed in this paper has been validated on a  Universal Robot 10e manipulator, equipped with an on-board 6-axis force/torque (F/T) sensor. Both the robot and the sensor run with a sampling time of \textit{2ms}.\\
The robot is employed to accomplish a set of tasks, including obstacle avoidance, joint control and position control of the end effector, as well as the satisfaction of the passivity constraint and the implementation of a variable admittance dynamics. Each task is encoded using a specific CBF, following the formulations exposed in Sec.~\ref{sec:CBF}.\\
The obstacle avoidance task is encoded by the following CBF:
\begin{equation}\label{eq: safety CBF}
    h_{safe} = -\xi_{s}(d^{2} - D_{min}^{2})
\end{equation}
in which $\xi_{s} = 10$ and $D_{min}$ is the minimum distance value between the obstacle and the tip of the end-effector, while $d = d(x, t)$ is defined as
\begin{equation}\label{eq: distance definition}
    d = ||x - x_{obs}||^{2}
\end{equation}
in which $x_{obs}$ is the Cartesian position of the closest obstacle.\\
Following the optimization-based formulation \eqref{eq: prob CBF multitask slack prio passive adm penalty}, for each CBF the correspondent gain $\xi_{(.)}$ is tuned and, after simple computations, the resulting constraint is obtained and inserted into the optimization problem. For the sake of simplicity, we chose the function $\alpha$ to be the identity function, such that $\alpha(h_{m}(\sigma,t)) = h_{m}(\sigma,t) \, \forall m \in \{ 1,\dots, M\}$. In the case of \eqref{eq: safety CBF}, for example, the resulting constraint is
\begin{equation}\label{eq: safety CBF constraint}
    2(d - D_{min})J(q)\dot{q} \geq -(h_{safe}(x,t)) + \delta_{s}
\end{equation}
in which $\delta_{s}$ is the dedicated slack variable.
The joint control task consists in maintaining the joint variables within an upper and lower limit. For each $i-$th joint, a dedicated CBF has been implemented as:
\begin{equation}\label{eq: joint limit CBF}
     h_{lim_{i}} = \xi_{l}\frac{(q^{+}_{i} - q_{i})(q_{i}- q^{-}_{i})}{(q^{+}_{i} - q^{-}_{i})} \quad i \in  \{ 1, \dots, n\}
\end{equation}
in which $\xi_{l} = 1$, while $q^{+}_{i}$ and $q^{-}_{i}$ represent the real joint limits of the $i$-th joint, in the joint space. Observe how ensuring that $h_{lim_{i}} \geq 0$ is equivalent to keeping the value of $q_{i} \in [q^{+}_{i}, q^{-}_{i}]$. 
\\
Finally, the position control task is encoded through the following CBF:
\begin{equation}\label{eq: pos control CBF}
    h_{pos} = -\xi_{p}||x - x_{goal}||^{2}
\end{equation}
in which $\xi_{p} = 5$ and $x_{goal}$ is the desired Cartesian position of the end effector.\\ 
Additionally, the desired admittance dynamics presents a time-varying repulsive potential $P(x,t)$ which is centered in the current position of the obstacle. The generated force $\frac{\partial P}{\partial x}(t)$ grows as the robot approaches the obstacle, such that the human can perceive an hazardous area during the interaction. The formulation of the resulting force is as follows
\begin{equation} \label{eq: repulsive potential}
    \frac{\partial P}{\partial x}(t) = 
    \begin{cases}    
    K_{rep}\Bigl( \frac{1}{d(p,t)} - \frac{1}{D^{*}} \Bigr) \frac{p(t) - p_{obs}(t)}{d^{3}(p,t)} \,\,\, if \, d(p,t)<D^*\\
    0 \quad otherwise
    \end{cases}
\end{equation}
in which $D^{*} > D_{min}$ is the activation distance for the potential. As seen before, such a behavior is non-passive, and energy needs to be extracted from the tank for preserving a passive energy balance.\\
A series of experiments have been conducted, in order to concurrently validate each component of the architecture.\\
First, the robot is moved to a desired goal using \eqref{eq: pos control CBF}. Along the way, the operator interacts with it, moving it around in the workspace. As soon as the interaction stops, the robot switches to free-motion and resumes its motion towards the goal, unaffected by the previous interaction. Fig. \ref{fig:h_pos vs Fe} shows the evolution of \eqref{eq: pos control CBF} over time, together with the intensity of the force acting on the system.\\
Secondly, the same position control task is implemented, this time with the presence of a time-varying obstacle moving in the workspace. As visible in Fig. \ref{fig:h_pos vs d}, as soon as the distance $d$ approaches the minimum value $D_{min} = 0.25m$, the constraint forces the robot to move away from the obstacle, preserving the overall safety. Once the obstacle is removed, the robot resumes its motion towards the goal.\\
Finally, the passivity of the interaction is validated using a variable admittance control. The admittance model encompasses a constant mass and damping term $M(x_{a})$ and $D(x_{a})$, as well as the time-varying potential in \eqref{eq: repulsive potential}. In particular, the desired inertia is chosen as a diagonal matrix with elements equal to $0.75kg$ and $0.25kg$ for translations and rotations respectively, while the damping is a diagonal matrix with constant elements equal to $0.05\frac{Ns}{m}$ (translations) and $0.025\frac{Ns}{m}$ (rotations).\\
During the experiment, the operator tries to guide the robot towards the obstacle. As the value of $d$ drops below $D^*=0.5m$, the operator feels the repulsive force \eqref{eq: repulsive potential} generated by the potential. Fig. \ref{fig:tank evolution} shows the evolution of the energy in the tank: as soon as the current value of $T(x_{t}(t))$ approaches $\underline{\varepsilon} = 0.1$, the constraint \eqref{eq: passivity constraint} ensures that the tank is never depleted by implementing a passive approximation of the desired behaviour.\\
Additional graphs are shown in the accompanying video, in which the three experiments are portrayed.
\begin{figure}
	\centering
	\includegraphics[width=\columnwidth]{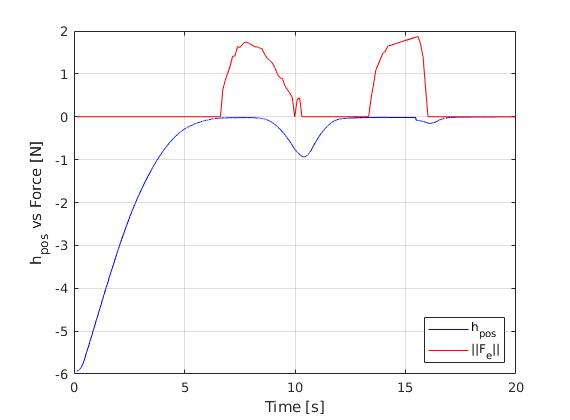}
	\caption{Value of the CBF $h_{pos}$ over time and norm of the force measured by the F/T sensor.}
	\label{fig:h_pos vs Fe}
\end{figure}
\begin{figure}
	\centering
	\includegraphics[width=\columnwidth]{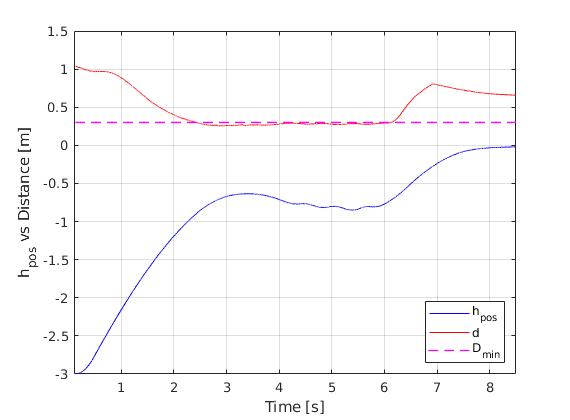}
	\caption{Value of the CBF $h_{pos}$ over time and distance between the end effector of the robot and the obstacle.}
	\label{fig:h_pos vs d}
\end{figure}
\begin{figure}
	\centering
\includegraphics[width=\columnwidth]{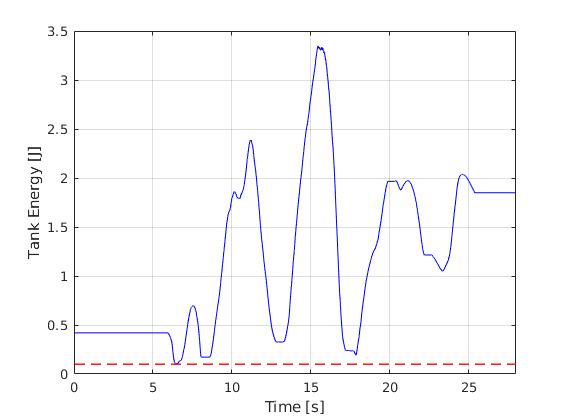}
	\caption{Evolution of the tank energy over time.}
	\label{fig:tank evolution}
\end{figure}
\section{Conclusions and Future Work}
\label{sec:conclusions}
In this paper we have proposed an optimization-based framework which combines CBFs and energy tanks for achieving a robust and flexible interaction, while respecting a set of time-varying constraints, thus guaranteeing the complete safety of the robotic application in a HRC scenario. We have formulated passivity as a constraint and inserted it, along other safety-relevant constraints, into a single convex optimization problem. Future work aims at establishing priorities among the different constraints in an optimal fashion and predicting the evolution of the energy in the tank.
\bibliographystyle{IEEEtran}
\bibliography{mybib.bib}

\end{document}